\documentclass{article} 
\usepackage{iclr2020_conference,times}

\usepackage{amsthm}
\usepackage{amsmath}
\usepackage{amsfonts}
\usepackage{mathtools}
\usepackage{tikz-cd}
\usepackage{float}

\usepackage{todonotes}
\usepackage{hyperref}
\usepackage{url}

\newtheorem{theorem}{Theorem}
\newtheorem{proposition}[theorem]{Proposition}
\newtheorem{lemma}[theorem]{Lemma}
\theoremstyle{definition}
\newtheorem{definition}[theorem]{Definition}
\newtheorem{remark}[theorem]{Remark}

\title{On the space-time expressivity of ResNets}

\author{Johannes M{\"u}ller\\
Max Planck Insitute for Mathematics in the Sciences\\
\texttt{jmueller@mis.mpg.de}
}

\iclrfinalcopy 
\begin{document}

\maketitle

\begin{abstract}
Residual networks (ResNets) are a deep learning architecture that substantially improved the state of the art performance in certain supervised learning tasks. Since then, they have received continuously growing attention. ResNets have a recursive structure \(x_{k+1} = x_k + R_k(x_k)\) where \(R_k\) is a neural network called a residual block. This structure can be seen as the Euler discretisation of an associated ordinary differential equation (ODE) which is called a neural ODE. Recently, ResNets were proposed as the space-time approximation of ODEs which are not of this neural type. To elaborate this connection we show that by increasing the number of residual blocks as well as their expressivity the solution of an arbitrary ODE can be approximated in space and time simultaneously by deep ReLU ResNets. Further, we derive estimates on the complexity of the residual blocks required to obtain a prescribed accuracy under certain regularity assumptions.
\end{abstract}

\section{Introduction}

Various neural network based methods have been proposed for the numerical analysis of partial differential equations (PDEs) \cite[see][]{lee1990neural,dissanayake1994neural,takeuchi1994neural,lagaris1998artificial} as well as for ordinary differential equations (ODEs) \cite[see][]{meade1994numerical,meade1994solution,lagaris1998artificial,breen2019newton}. In subsequent years those methods where improved and extended to a variety of settings and we refer to \cite{yadav2015introduction} for an overview of neural network based methods for ODEs. Recently, deep networks have successfully been applied to the numerical simulation of stationary and non stationary PDEs by \cite{weinan2018deep} and \cite{weinan2017deep, han2018solving} respectively; a list of further improvement of those methods can be found in \cite{grohs2019spacetime}. The promising empirical performance of those approaches raised interest in theoretical guarantees and led to a number of error estimates \cite[see][]{jentzen2018proof, han2018convergence, grohs2018proof, elbrachter2018dnn, berner2018analysis, reisinger2019rectified, kutyniok2019theoretical}. In particular it can be shown that neural networks are capable of approximating the solutions of a number of PDEs without suffering from the curse of dimensionality. However, it should be noted that those works only provide estimates for the spatial error at a fixed time rather than the approximation error in space and time simultaneously.

Compared to the case of PDEs the analysis of the approximation error for ODEs is less complete. Although a priori and a posteriori error estimates are present in the literature \cite[see][respectively]{filici2008neural,filici2010error} they only consider the solution for a single initial value rather than the full space-time solution
\begin{equation}\label{spacetimesol}
     x(0, y) = y, \quad \partial_t x(t, y) = f(t, x(t, y)) \quad \text{for all } t, y
\end{equation}
to the right hand side \(f\). Recently, \cite{grohs2019spacetime} established an approximation result in space-time and showed that Euler discretisations of a certain class of neural ODEs can be approximated by neural networks with error decreasing exponentially in the complexity of the networks. Those are the first space-time error estimates in the study of neural network based methods for either PDEs or ODEs. Yet, in order to obtatin space-time error estimates for the solution of an ODE one has to bound the approximation error of the class of Euler discretisations to the solution of this ODE. Such estimates are implied by our main result Theorem \ref{uniapproxcomplex} concerning the approximation of space-time solutions with residual networks. However, there is a further motivation in the study of the approximation capabilities of residual networks that we will present now.

\subsection*{Residual networks and dynamical systems}

Residual networks (ResNets) make use of skip connections which were introduced to overcome difficulties in the training of deep neural networks in supervised learning tasks. Rather than using the iterative scheme \(x_{l+1} \coloneqq\rho(A_lx_l+b_l)\) like a traditional feedforward network, ResNets copy the input \(x_l\) to some subsequent layer, in the easiest case to the following layer which leads to
\begin{equation}\label{resnet}
x_{l+1} \coloneqq x_l + \rho(A_lx_l+b_l) \quad \text{for } l = 0, \dots, L-1.
\end{equation}
Obviously, this is only well defined if the dimensions of all states \(x_l\) agree. It was shown in \cite{he2016deep} that ResNets  are superior to traditional feedforward neural networks in some image classification tasks. It has been pointed out in \cite{haber2018learning} that the recursive structure \eqref{resnet} can be interpreted as the explicit Euler discretisation of the ordinary differential equation
\begin{equation}\label{ODEresnet}
\partial_t x(t) = \rho\big(A(t)x(t) + b(t)\big).
\end{equation}
Building on this observation \cite{haber2017stable} transferred the knowledge about the stability of ODEs to the stability of forward propagation in ResNets and \cite{lu2017beyond} introduced neural networks corresponding to other numerical schemes for ODEs like implicit Euler or Runge-Kutta schemes. Further, \cite{chen2018neural} replaced ResNets by ODEs in supervised learning tasks and achieved state of the art performance with fewer parameters. A rigorous justification for this approach using the notion of \(\Gamma\)-convergence was established in \cite{thorpe2018deep}. Lately, ResNets have been proposed in \cite{rousseau2019residual} as an approximation of space-time solutions of a much more general class of ODEs than \eqref{ODEresnet} which always admits non decreasing solutions. Further, this was applied to the problem of diffeomorphic image registration which can be interpreted as a controlled ODE problem.

The expressivity of ResNets was studied in different ways. It was shown by \cite{lin2018resnet} that ResNets are able to approximate arbitrary \(L^p\)-functions and \cite{cuchiero2019deep} showed that ResNets can take prescribed values on arbitrary point sets. Both works consider the input-output mapping \(x_0\mapsto x_L\) induced by a ResNet. Similarly, \cite{dupont2019augmented} and \cite{zhang2019approximation} studied the approximation capabilities of neural ODEs at final time. Although many works perceive ResNets as discrete dynamical systems \cite[see][and subsequent work]{weinan2017proposal, liu2019selection} an analysis of the expressivity of their dynamics is still absent.

\subsubsection*{Contributions}

We study the expressivity of the dynamics of ResNets and show that ResNets can approximate solutions of arbitrary ODEs in space-time. This includes
the solution of the control problem ResNets where proposed for in \cite{rousseau2019residual}. More precisely, we make the following contributions:
\begin{enumerate}
    \item \emph{Universality}: ResNets can approximate solutions of arbitrary ODEs uniformly in space and time simultaneously (see Theorem \ref{uniapproxflow}).
    \item \emph{Complexity bounds}: Assume that the right hand side \(f\) is Lipschitz continuous. Then the solution to this ODE can be approximated with (local) error \(\mathcal O(n^{-1})\) through ResNets with \(n\) residual blocks which have \(\mathcal O(r_n^d n^{d})\) neurons; here, \((r_n)_{n\in\mathbb N}\subseteq(0, \infty)\) is an arbitrary sequence diverging to \(+\infty\) and \(d\) is the dimension of the ODE (see Theorem \ref{uniapproxcomplex}).
\end{enumerate}

\section{Definitions and notation}

Let for the remainder \(d, m, L\) be natural numbers. Further, we consider tupels 
\[\theta = \left( (A_1, b_1), \dots, (A_L, b_L)\right)\]
of matrix-vector pairs where \(A_l\in\mathbb R^{N_{l}\times N_{l-1}}\) and \(b_l\in\mathbb R^{N_l}\) and \(N_0 = d, N_L = m\). Every matrix vector pair \((A_l, b_l)\) induces an affine linear transformation that we denote by \(T_l\colon \mathbb R^{N_{l-1}} \to\mathbb R^{N_l}\). The \emph{neural network with parameters} \(\theta\) and with respect to some \emph{activation function} \(\rho\colon\mathbb R\to\mathbb R\) is the function
\[R=R_\theta\colon\mathbb R^d\to\mathbb R^m, \quad x\mapsto T_L(\rho(T_{L-1}(\rho(\cdots \rho(T_1(x)))))),\]
where \(\rho\) is applied componentwise. We call \(d\) the \emph{input} and \(m\) the \emph{output dimension}, \(L\) the \emph{depth} and \(N(\theta)\coloneqq\sum_{l=0}^L N_l\) the \emph{number of neurons} of the network. If we have \(f=R_\theta\) for some \(\theta\) we say that the function \(f\) is \emph{expressed} by the neural network.

In the following we restrict ourselves to the case of a specific activation function which is not only commonly used in practice \cite[see][]{ramachandran2017searching} but also exhibits nice theoretical properties \cite[see][]{arora2016understanding,petersen2018topological}. The \emph{rectified linear unit} or \emph{ReLU activation function} is defined via \(x\mapsto \max\left\{ 0, x\right\}\) and we call networks with this activation \emph{ReLU networks}.

Finally, we introduce the notion of residual networks. In order to interpret ResNets as functions in space-time we define them to be Euler discretisation of a certain class of ODEs which are linearly interpolated in time. It is important to note that this might differ from other definitions of residual networks present in the literature.

\begin{definition}[Residual network]
Let \(\theta = (\theta_1, \dots, \theta_n)\) be a tupel of parameters of neural networks with input and output dimension \(d\). Let \(R_1, \dots, R_{n}\colon\mathbb R^d\to\mathbb R^d\) denote the neural networks with parameters \(\theta_1, \dots, \theta_{n}\) and some activation \(\rho\). We refer to those networks as \emph{residual blocks}.
The \emph{residual network} or \emph{ResNet} \(x^n\colon[0, 1]\times\mathbb R^d\to\mathbb R^d\) with parameters \(\theta = (\theta_1, \dots, \theta_n)\) and with respect to the activation function \(\rho\) is defined via
\[x^n(0, y) \coloneqq y, \quad x^n(t_{k+1}, y)\coloneqq x^n(t_k, y) + n^{-1}\cdot R_{k+1}(x^n(t_k, y))\]
for \(k = 0, \dots, n-1\) and linearly in between. In the remainder, we will only consider ResNets with respect to the ReLU activation function and call those \emph{ReLU ResNets}.
\end{definition}

\section{Presentation of the main results}

Now we have introduced enough notation to state our main results precisely.

\begin{theorem}[Space-time approximation with ResNets]\label{uniapproxflow}
Let \(d\in\mathbb N, f\in L^1([0, 1]; \mathcal C^{0, 1}_b(\mathbb R^d; \mathbb R^d))\)\footnote{Up to a technical measurability property (Bochner measurability) this means that \(f(t, \cdot)\) is bounded and Lipschitz continuous for almost all \(t\) and that the uniform norm and Lipschitz constants are integrable.} 
and let \(x\) be the space-time solution\footnote{see \eqref{spacetimesol}; we use the notion of weak solutions introduced in the appendix; if \(f\) is continuous this coincides with the classical notion of a solution; further, the ODE is globally well posed for this class of right hand sides.} of the ODE with right hand side \(f\). Then for every compact set \(K\subseteq\mathbb R^d\) and \(\varepsilon>0\) there is a ReLU ResNet \(\tilde x\) such that
\[ \left\lVert \tilde x(t, y) - x(t, y) \right\rVert \le \varepsilon \quad \text{for all } t\in [0, 1], y\in K.\]
\end{theorem}

The proof is based on the observation that \(f\) can be approximated by functions that are piecewise constant in time on the intervals \([k/n, (k+1)/n)\). By standard continuity results for the solution operator of ODEs the approximation also holds for the associated space-time solutions and thus one can without loss of generality assume that \(f\) is piecewise constant. However, if \(f\) is merely integrable in time it is not possible to bound the number of constant regions. Hence, one can not bound the number of the residual blocks that is required in order to achieve a prescribed approximation accuracy under no temporal regularity assumptions. Nevertheless, in the proof of the result the approximation in space and in time are clearly separated and in fact the constructed residual blocks share weights depending on the temporal regularity of the right hand side. Hence, the same arguments can be used to establish bounds on the complexity of the residual networks like the following. 

\begin{theorem}[Space-time approximation with complexity bounds]\label{uniapproxcomplex}
Let \(d\in\mathbb N\), \((r_n)_{n\in\mathbb N}\subseteq(0, \infty)\) be a sequence diverging to \(+\infty\) and let \(f\colon[0, 1]\times\mathbb R^d\to\mathbb R^d\) be a bounded and 
Lipschitz continuous function. Let \(x\colon[0, 1]\times\mathbb R^d\to\mathbb R^d\) be the space-time solution of the ODE with right hand side \(f\). Then for every \(n\in\mathbb N\) there is a ReLU ResNet \(x^n\) with parameters \(\theta^n = (\theta^n_1, \dots, \theta^n_n)\) such that the following are satisfied:
\begin{enumerate}
\item \emph{Approximation:} For every compact set \(K\subseteq\mathbb R\) it holds
\[\sup_{t\in[0, 1], y\in K}\left\lVert x^n(t, y) - x(t, y) \right\rVert \in \mathcal O(n^{-1}).\]
\item \emph{Complexity bounds:} Every residual block \(\theta^n_k\) has depth \(\big\lceil \log_2((d+1)!)\big\rceil +2\) and satisfies
\[N(\theta^n_k) \in \mathcal O\left(r_n^d n^{d}\right).\]
Finally, all but \(\mathcal O\left(r_n^d n^{d}\right)\) weights can be fixed.
\end{enumerate}
\end{theorem}

We shall note that in similar fashion further complexity estimates can be obtained if the temporal and spatial regularities are different. This includes cases where the Lipschitz constants in time and space differ or where one is given by some Sobolev or smoothness property; the resulting complexity bounds would change accordingly to the respective spacial and temporal approximation results.

\subsection*{Outline of the proof}

In a nutshell the proof of the approximation results presented above relies on a combination of a spacial approximation result for ReLU networks and a Gr{\"o}nwall argument. We quickly present the key arguments of Theorem \ref{uniapproxcomplex} and postpone any rigorous calculations to the appendix.

The proof is based on a variant of the universal approximation results in \cite{hanin2017universal} and \cite{yarotsky2018optimal} and we follow \cite{he2018relu} for the construction of piecewise linear interpolations. This method achieves optimal rates under the assumption of continuous weight assignment which are also optimal for bounded depth networks \cite[see][]{devore1989optimal, yarotsky2018optimal}. Although faster approximation rates for deep networks of bounded width are established in \cite{yarotsky2018optimal} we use the following result as it allows a direct control of the uniform norm of the networks. However, our arguments can be generalised to other universal approximation results.

\begin{proposition}[Universal approximation under Lipschitz condition]\label{uvreplip}
Let \(d, m\in\mathbb N\) and \(r>0\) and let \(f\colon\mathbb R^d\to\mathbb R^m\) be Lipschitz continuous. Then for every \(\varepsilon>0\) there is a ReLU network \(R_\varepsilon\) with parameters \(\theta_\varepsilon\) that satisfies the following:
\begin{enumerate}
\item \emph{Approximation:} It holds that
\(\sup_{x\in[-r, r]^d} \left\lVert f(x) - R_\varepsilon(x) \right\rVert\le\varepsilon.\)
\item \emph{Complexity bounds:} The network has depth \(\big\lceil \log_2((d+1)!)\big\rceil +2\), \(\mathcal O\left(r^d\varepsilon^{-d}\right)\) many neurons and all but \(\mathcal O \left(r^d\varepsilon^{-d}\right)\) weights can be fixed. Finally, if \(\left\lVert f \right\rVert\) is bounded by \(c\) so is \(\left\lVert R_\varepsilon \right\rVert\).
\end{enumerate}
\end{proposition}

\begin{proof}[Proof of Theorem \ref{uniapproxcomplex}]
For every \(n\in\mathbb N\) the previous proposition yields the existence of neural networks \(R^n_1, \dots, R^n_n\) of asserted complexity that satisfy
\[\sup_{x\in[-r_n, r_n]^d} \big\lVert f(t_k, x) - R^n_{k+1}(x) \big\rVert \le n^{-1}.\]
Since \(f\) is bounded, let's say by \(c>0\)\footnote{By this we mean that the (Euclidean) norm is bounded by \(c\).}, so are all realisations \(R^n_k\) independent of \(k\) and \(n\). Hence, for any initial condition \(y\in B_R\) in some ball the true solution \(x(t, y)\) as well as the ResNet \(x^n(t, y)\) arising from the networks \(R^n_1, \dots, R^n_n\) remain in the bounded set \(B_{R+c}\). However, on this bounded set the realisations \(R^n_k\) approximate the right hand side uniformly and thus every ResNet can be interpreted as an perturbed Euler discretisation of the ODE with right hand side \(f\). Therefore, the residual network satisfies an integral equation for every fixed inital value \(y\). An application of  Gr{\"o}nwall's inequality yields that \(x^n\) does in fact converge towards \(x\) uniformly on \([0, 1]\times B_R\) with approximation error in \(\mathcal O(n^{-1})\). Since the ball \(B_R\) was arbitrary the general statement follows. 
\end{proof}

\section{Discussion and further research}

We showed that residual networks are capable of approximating the solution of general ODEs in space-time. Further, under additional regularity assumptions we established bounds on the complexity of the residual blocks. The arguments presented above can directly be generalised to other classes of right hand sides \(f\) that allow a more effective spatial approximation through neural networks. This includes compositional functions or classes of (piecewise) smooth functions \cite[see][]{mhaskar2016learning, liang2016deep, petersen2018optimal, yarotsky2018optimal, shen2019nonlinear, montanelli2019error}.

For future directions we propose to investigate whether and if so in what notion controlled ResNets converge towards controlled ODEs. It would be particularly interesting to see which weight regularisation corresponds to which regularisations of controlled ODEs. This would be a continuation of the work by  \cite{thorpe2018deep, avelin2019neural} where residual blocks of one layer and constant weights are studied. Further, it is not clear for which class of controlled ODEs the curse of dimensionality can be circumvented, faster approximation rates can be established or weights can be shared between different residual blocks. 

\clearpage

\section*{Acknowledgments}
JM acknowledges support by the Evangelisches Studienwerk Villigst e.V. and the IMPRS MiS. Further, the authors want to thank Nihat Ay, Nicolas Charon, Philipp Harms, Jasper Hofmann, Guido Mont\'{u}far and Hsi-Wei Hsieh for valuable comments and discussions.

\bibliography{iclr2020_conference}
\bibliographystyle{abbrvnat}

\clearpage

\appendix
\section{Universal approximation with ReLU networks}

This section is concerned with the proof of the universal approximation result in Proposition \ref{uvreplip}. Similar proofs relying on the approximation through interpolation can be found in \cite{hanin2017universal,yarotsky2018optimal}. We also follow \cite{he2018relu} for the expression of nodal basis functions, however, we also bound the complexity of the ReLU networks needed to express such functions.

\subsection{Triangulations and piecewise linear functions}

Let in the following \(\mathcal T\) be a \emph{locally finite triangulation} of the entire Euclidean space \(\mathbb R^d\) consisting of nondegenerate \(d+1\) simplices \(\left\{ \tau_k\right\}_{k\in\mathbb N}\) and vertices \(\mathcal V\). More precisely, this means that the union of the simplices covers the entire space but that their interiors are pairwise disjoint and that every bounded set only intersects with finitely many simplices. Further, every simplex should be the convex hull of \(d+1\) points and have non trivial interior.

For a vertex \(x\in\mathcal V\) we set \(N(x)\coloneqq \left\{ k\in\mathbb N\mid x\in\tau_k\right\}\) define the \emph{maximum number of neighboring simplices} to be 
\[k_{\mathcal T}\coloneqq \sup_{x\in\mathcal V} \left\lvert N(x) \right\rvert \]
which we will assume to be finite. Further, we set \[\Omega(x)\coloneqq \bigcup_{k\in N(x)} \tau_k\] and call \(\mathcal T\) \emph{locally convex}, if \(\Omega(x)\) is convex for all \(x\in\mathcal V\).
The \emph{fineness} of the triangulation is defined to be the supremum over the diameters of the simplices
\[\left\lvert \mathcal T \right\rvert\coloneqq \sup_{k\in\mathbb N} \operatorname{diam}(\tau_k) \]
and we will assume that is finite. We will later give an explicit construction of a triangulation that satisfies those conditions. 

\begin{definition}[Piecewise linear functions]
We say a function \(f \colon\mathbb R^d \to\mathbb R\) is \emph{piecewise linear (PWL) with respect to} \(\mathcal T\) if it is affine linear on every simplex of the triangulation. Given such a function \(f\) we call
\[ \left\lvert \mathcal V(f) \right\rvert \coloneqq \left\lvert \left\{ x\in\mathcal V\mid f(x)\ne 0\right\} \right\rvert\]
the \emph{degrees of freedom} of the function.
\end{definition}

Note that the definition of PWL functions automatically implies continuity since the affine regions are closed and cover \(\mathbb R^d\) and affine functions are continuous. It is well known from the theory of finite elements that for every vertex \(x\in\mathcal V\) there is a with respect to \(\mathcal T\) piecewise linear function \(\phi\) that satisfies \(\phi(x)=1\) and vanishies at every other vertex. We call this function the \emph{nodal basis function} associated with \(x\). The nodal basis functions form a basis of the space of PWL functions with finitely many degrees of freedom.

We will give an explicit construction of a triangulation that satisfies the assumptions from above. For this note that the unit cube \([0, 1]^d\) can be divided into the simplices
\[S_\sigma\coloneqq\Big\{ x\in\mathbb R^d \;\big\lvert\; 0\le x_{\sigma(1)} \le \cdots \le x_{\sigma(d)} \le 1\Big\}\]
where \(\sigma\) is a permutation of the set \(\left\{ 1, \dots, d\right\}\). It is straight forward to check that those simplices cover the unit cube and have disjoint interiors and are non degenerate \(d+1\) simplices. The fineness of this triangulation is \(\sqrt d\). We call this triangulation the \emph{standard triangulation} of the Euclidean space \(\mathbb R^d\).

We will need the fact that the standard triangulation is locally convex. Since it is periodic, it suffices to show that \(\Omega(0)\) is convex. In order to do this we will show that
\[\Omega(0) = \left\{ z\in[-1, 1]^d \;\Big\lvert\; z_i\le z_j +1 \text{ for all } i, j = 1, \dots, d \right\} \eqqcolon A.\]
This expresses \(\Omega(0)\) as an intersection of convex sets and hence shows the convexity of \(\Omega(0)\). 

Let us take \(z = x - y\in \Omega(0)\) with \(x, y\in S_\sigma\) where \(y\) has binary entries. Then we obviously have \(z\in [-1, 1]^d\). Let now \(i, j\in \left\{ 1, \dots, d\right\}\), then we have to distinguish two cases. The first one is \(\sigma^{-1}(i) \le \sigma^{-1}(j)\) which implies \(x_i\le x_j\) and \(y_i\le y_j\) and thus
\[z_i - z_j = (x_i - x_j)+(y_j - y_i) \le y_j \le 1.\]
For \(\sigma^{-1}(i) >\sigma^{-1}(j)\) an analogue computation shows \(z_i\le z_j+1\) and hence we obtain the inclusion \(\Omega(0)\subseteq A\).
To see that the other inclusion holds true, we fix \(z\in A\) and set \(I\coloneqq\left\{ i\mid z_i\ge0\right\}\) and \(J\coloneqq\left\{ 1, \dots, d\right\}\setminus I\). Further, we define \(y\in \left\{ 0, 1\right\}^d\) via
\[y_i\coloneqq\begin{cases} \; 0 \quad &\text{for } i\in I \\ \; 1 &\text{otherwise}\end{cases}\]
and \(x\coloneqq z + y\). By construction we have \(z = x - y\) and \(x, y\in[0, 1]^d, y\in\left\{ 0, 1\right\}^d\) and hence we only need to show the existence of a permutation \(\sigma\) such that \(x, y\in S_\sigma\). Obviously, the statement \(y\in S_\sigma\) is equivalent to \(\sigma^{-1}(i)\le\sigma^{-1}(j)\) for all \(i\in I, j\in J\). Since for \(i\in I\) and \(j\in J\) we have
\[x_i = z_i \le z_j + 1= x_j, \]
there is a permutation that additionally satisfies \(\sigma^{-1}(i)\le \sigma^{-1}(j)\) whenever \(x_i\ge x_j\) for some \(i, j\in\left\{ 1, \dots, d\right\}\).

\subsection{Exact expression of piecewise linear functions as ReLU networks}

We quickly present well known examples of functions that can exactly be expressed by ReLU networks \cite[see][]{he2018relu, petersen2018optimal}.

\begin{enumerate}
    \item \emph{Identity mapping.} A basic calculation shows the identity 
\begin{equation}
x = \rho(x) - \rho(-x) \quad \text{for all } x\in\mathbb R^d.
\end{equation}
Hence, the identity is can be expressed as a ReLU network of width \(2d\) which is visualised below. Note, that one can express the identity function as arbitrarily deep ReLU networks of width \(2d\) since one can simply add more hidden layers where the affine linear transformation is the identity. 
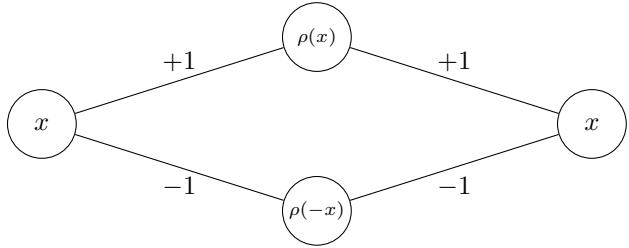
\begin{figure}[H]
\centering
\begin{tikzpicture}\tikzset{node distance = 0.5cm and 3cm}
\tikzstyle{vertex}=[draw,circle,minimum size=26pt,inner sep=1pt]
\node[vertex] (input) at (0,0) {\(x\)};
\node[vertex, above right = of input] (x1) {\(\scriptstyle\rho(x)\)};
\node[vertex, below right = of input] (x2) {\(\scriptstyle\rho(-x)\)};
\node[vertex, below right = of x1] (x3) {\(x\)};
\draw (input) -- node [above] {\(+1\)} (x1);
\draw (input) -- node [below] {\(-1\)} (x2);
\draw (x1) -- node [above] {\(+1\)} (x3);
\draw (x2) -- node [below] {\(-1\)} (x3);
\end{tikzpicture}
\caption{An example for an expression of the identity mapping as a ReLU network.}\label{examplegraphs}
\end{figure}
Similarly, one obtains that the absolute value can be expressed as a ReLU network of arbitrary depth and width \(2\) since \(\left\lvert x \right\rvert = \rho(x) + \rho(-x)\).
\item \emph{Minimum operation. } It is elementary to check
\[\min(x, y) = \frac12\big( x + y - \left\lvert x - y \right\rvert\big).\]
We have already seen how the terms on the right hand side can be expressed as shallow ReLU networks and hence we obtain
\begin{equation}\label{exmin}
    \min(x, y) = \frac12 \Big( \rho(x+y) - \rho(-x-y) - \rho(x-y) - \rho(-x+y)\Big).
\end{equation}

Therefore, the minimum operation can be expressed as a shallow ReLU network of width \(4\) and with weights \(\pm\frac12, \pm1\).

\begin{figure}[H]
\centering
\begin{tikzpicture}\tikzset{node distance = 0.4cm and 2cm}
\tikzstyle{vertex}=[draw,circle,minimum size=38pt,inner sep=1pt]
\node[vertex] (input1) at (0,0) {\(x\)};
\node[vertex, below = of input1] (input2) {\(y\)};
\node[vertex, right = of input1] (y2) {\(\scriptstyle \rho(-x-y)\)};
\node[vertex, above = of y2] (y1) {\(\scriptstyle \rho(x+y)\)};
\node[vertex, right = of input2] (y3) {\(\scriptstyle \rho(x-y)\)};
\node[vertex, below = of y3] (y4) {\(\scriptstyle \rho(-x+y)\)};
\node[vertex, below right = -0.2cm and 2cm of y2] (z) {\(\scriptstyle \min(x, y)\)};
\foreach \i in {1,...,4} {
\draw (input1) -- (y\i);
\draw (input2) -- (y\i);
\draw (y\i) -- (z);
}
\end{tikzpicture}
\caption{Expressing the minimum operation through a shallow ReLU network.}
\end{figure}
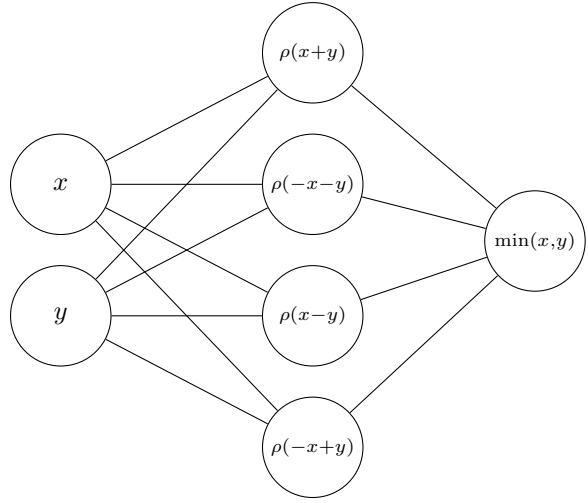
\end{enumerate}

A consequence of the fact that the identity can be expressed as a shallow network is that the class of neural networks is closed under summation and also parallelisation. We will use those concepts of parallelisation and summation of networks which are relatively intuitive and we refer to \cite{petersen2018optimal} for further details.

The expression of nodal basis functions as ReLU networks relies on the following proposition. 

\begin{lemma}
Let \(\mathcal T\) be a locally finite and locally convex triangulation of \(\mathbb R^d\) and let \(x\in\mathcal V\) with nodal basis function \(\phi\). Then we have
\begin{equation}\label{nbfmin}
    \phi(y) = \max\left\{ 0, \min_{k\in N(x)} g_k(y) \right\} = \min_{k\in N(x)} \rho(g_k(y)) \quad \text{for all } y\in\mathbb R^d,
\end{equation}
where \(g_k\) is the globally affine linear function that agrees with \(\phi\) on the simplex \(\tau_k\).
\end{lemma}

For a proof we refer to \cite{he2018relu} which we also follow closely for the next two results, however, we additionally bound the complexity of the neural networks.

\begin{proposition}[Minimum function]\label{mfNN}
The minimum function \(\min\colon\mathbb R^d\to\mathbb R\) can be expressed through a ReLU network of depth \(\lceil \log_2(d)\rceil+1\). Further, such a network can be constructed with weights \(\big\{ 0, \pm\frac12, \pm1\big\}\) and \( \mathcal O(d)\) many neurons and \(\mathcal O(d)\) non-zero weights. 
\end{proposition}
\begin{proof}
Let us for the sake of easy notation assume \(d = 2^m\). 
The construction of the representation of the minimum function relies on the observation that the minimum operation is the composition of \(\log_2(d)=m\) mappings of the form
\[f_k\colon\mathbb R^{2^k} \to \mathbb R^{2^{k-1}}, \quad \begin{pmatrix}
x_1 \\ \vdots \\ x_{2^k}
\end{pmatrix}\mapsto \begin{pmatrix}
\min(x_1, x_2) \\ \vdots \\ \min(x_{2^k-1}, x_{2^k})
\end{pmatrix}. \]
Those functions are the realisation of a parallelisation of the representation of the minimum function constructed in Example \ref{exmin}. More precisely, \(f_k\) can be represented through a shallow ReLU network where the dimension of the hidden layer is \(2\cdot2^k\). The concatenation of the \(m\) networks that represent the functions \(f_k\) is a representation of the minimum function of depth \(m+1\). By adding the dimensions of the layers we obtain the this network has
\[2^m+2\cdot 2^m + \dots + 2^2 + 1= 5d - 3\]
neurons and \(4\cdot (5d - 4)\) non-zero weights. 
\end{proof}

\begin{theorem}[Exact expression of PWL functions as ReLU networks]\label{exchar}
Consider \(d, m\in\mathbb N\) and let \(\mathcal T\) be a locally finite and locally convex triangulation of \(\mathbb R^d\) with \(k_{\mathcal T}<\infty\). Every function \(f\colon\mathbb R^d\to\mathbb R^m\) that is piecewise linear with respect to \(\mathcal T\) with \(N\) degrees of freedom can be expressed as a deep ReLU network with depth \(\lceil \log_2(k_{\mathcal T})\rceil + 2\) and at most \( \mathcal O(mk_{\mathcal T}N + d)\) neurons. Further, all but
\(m(d+1)k_{\mathcal T}N\)
weights can be fixed.
\end{theorem}
\begin{proof}
We assume \(m=1\) and note that the general statement follows from building a parallelised network.
Since \(f\) is the linear combination of \(N\) nodal basis functions \(\phi\) and hence it suffices to represent \(\phi\) through a neural network as a representation of \(f\) can be obtained by considering the standard addition of those networks.

In order to represent \(\phi\), we use \eqref{nbfmin} and the previous proposition. For the sake of easy notation we assume that \(N(x) = \left\{ 1, \dots, M\right\}\), then \(\phi\) can be represented by the following network depicted in Figure \ref{PWLNN} where the dashed part stands for a representation of the minimum function. It is clear that all weights except the ones of the first layer – which are \((d+1)M \le (d+1)k_{\mathcal T}\) many – are fixed. 
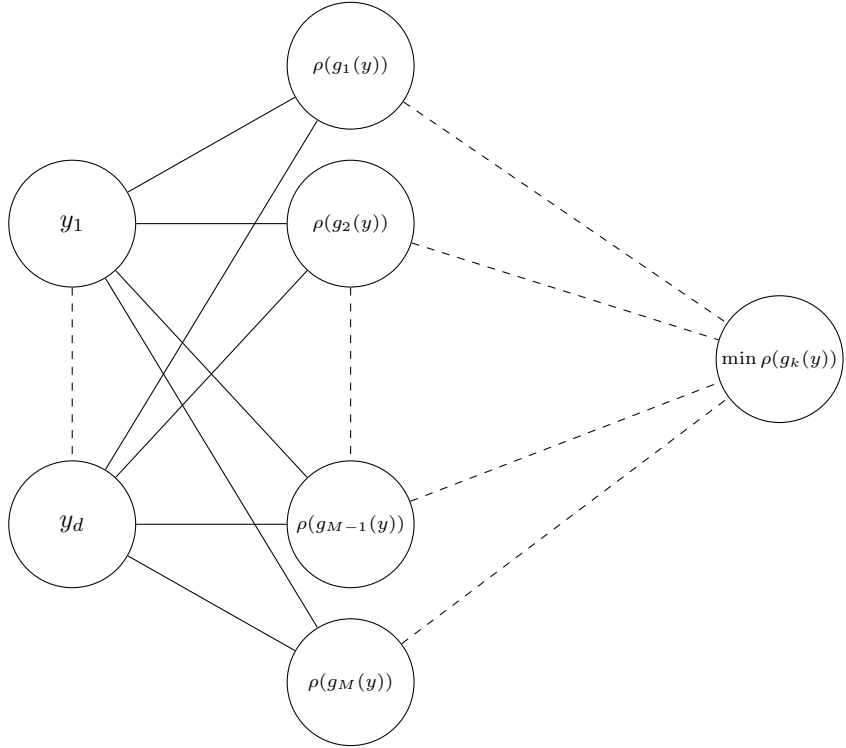
\begin{figure}[h!]
\centering
\begin{tikzpicture}\tikzset{node distance = 0.4cm and 2cm}
\tikzstyle{vertex}=[draw,circle,minimum size=48pt,inner sep=1pt]
\node[vertex] (input1) at (0,0) {\(y_1\)};
\node[vertex]
(input2) at (0, -4) {\(y_{d}\)};
\node[vertex, right = of input1] (y2) {\(\scriptstyle\rho(g_2(y))\)};
\node[vertex, above = of y2] (y1) {\(\scriptstyle\rho(g_1(y))\)};
\node[vertex, right = of input2] (y3) {\(\scriptstyle\rho(g_{M-1}(y))\)};
\node[vertex, below = of y3] (y4) {\(\scriptstyle\rho(g_M(y))\)};
\draw[dashed] (input1) -- (input2);
\draw[dashed] (y2) -- (y3);
\node[vertex, below right = 0.6cm and 4.5cm of y2] (output) {\(\scriptstyle\min \rho(g_k(y))\)};
\foreach \i in {1,...,4} {
\draw (input1) -- (y\i);
\draw (input2) -- (y\i);
\draw[dashed] (y\i) -- (output);
}
\end{tikzpicture}
\caption{An expression of a nodal basis function \(\phi\) where the dashed part stands for the expression of the minimum function that was constructed in Proposition \ref{mfNN}.}\label{PWLNN}
\end{figure}
\end{proof}

\subsection{Universal approximation through interpolation}

We introduce the \emph{modulus of continuity}
\[w_f\colon[0, \infty)\to[0, \infty], \quad  \delta\mapsto \sup\Big\{\left\lVert f(x) - f(y) \right\rVert \;\big\lvert\; x, y\in\Omega, \left\lVert x - y \right\rVert \le \delta \Big\}.\]
It is elementary to check that a function is uniformly continuous if and only if the modulus of continuity takes finite values and is continuous. The pseudoinverse \(w_f^{-1}\colon[0, \infty)\to[0, \infty)\) of the modulus of continuity is defined via
\[\varepsilon \mapsto \inf\left\{ \delta>0\mid w_f(\delta)>\varepsilon\right\}.\]
Note that if \(w_f\) is continuous, we have \(w_f(w_f^{-1}(\varepsilon))=\varepsilon\) for all \(\varepsilon>0\), i.e. we have 
\[\left\lVert f(x) - f(y) \right\rVert \le \varepsilon \quad \text{for all } x, y\in\Omega \text{ with } \left\lVert x - y \right\rVert \le w_f^{-1}(\varepsilon).\]
Finally, if \(f\) is Lipschitz continuous with constant \(L\) we have \(w_f(\delta)\le L\delta\) and hence \(w_f^{-1}(\varepsilon)\ge\frac\varepsilon L\).

\begin{proposition}[Function approximation with piecewise linear functions]
Let \(d, m\in\mathbb N\) and \(f\colon\mathbb R^d\to\mathbb R^m\) be a continuous function and \(\mathcal T\) be a locally finite triangulation of the Euclidean space \(\mathbb R^d\) with fineness \(\delta\in(0, \infty)\). Let \(\Omega\subseteq\mathbb R^d\) be a union of simplices of \(\mathcal T\) and \(g\) be the with respect to \(\mathcal T\) piecewise linear function that agrees with \(f\) on all vertices inside of \(\Omega\) and vanishes everywhere else. Then we have
\[\left\lVert f - g \right\rVert_{\infty, \Omega}\le w_{f|_\Omega}(\delta).\]
Finally, we have \(\left\lVert g \right\rVert_\infty\le \left\lVert f \right\rVert_\infty\).
\end{proposition}
\begin{proof}
Let \(x\in \Omega\) then \(x\) lies in a convex simplex with vertices \(x_1, \dots, x_{d+1}\in\Omega\). Hence, we find convex weights \(\alpha_1, \dots, \alpha_{d+1}\in[0, 1]\) such that \(x = \sum_{i=1}^{d+1}\alpha_ix_i\). Now we obtain
\[\big\lVert f(x) - g(x) \big\rVert = \left\lVert f(x) - \sum_{i=1}^{d+1}\alpha_i f(x_i) \right\rVert \le \sum_{i=1}^{d+1}\alpha_i \big\lVert f(x) - f(x_i) \big\rVert \le w_{f|_\Omega}(\delta). \]
Furthermore, if \(\left\lVert f \right\rVert\) is bounded by \(c\), then we obtain 
\[\big\lVert g(x) \big\rVert \le \sum_{i=1}^{d+1}\alpha_i \cdot \big\lVert f(x_i) \big\rVert \le \sum_{i=1}^{d+1}\alpha_i \cdot c = c \]
for all \(x \in \mathbb R^d\).
\end{proof}

Combining the previous results with the construction of the standard triangulation we obtain the following result.

\begin{proposition}[Universal approximation with ReLU networks]\label{app:uniapprox}
Consider a continuous function \(f\colon\mathbb R^d\to\mathbb R^m\) where \(d, m\in\mathbb N\). Let further \(r>0\) and \(\varepsilon>0\) and let \(w_{f, r}\) be the modulus of continuity of \(f|_{[-r, r]^d}\). Then for every \(\varepsilon>0\) there is a ReLU network \(R_\varepsilon\) with parameters \(\theta_\varepsilon\) that satisfies the following:
\begin{enumerate}
\item \emph{Approximation:} It holds that
\(\sup_{x\in[-r, r]^d} \left\lVert f(x) - R_\varepsilon(x) \right\rVert\le\varepsilon.\)
\item \emph{Complexity bounds:} The network has depth \(\big\lceil \log_2((d+1)!)\big\rceil +2\)
, \(\mathcal O\big(\omega_{f,r}^{-1}(\varepsilon)^{-d}\big)\) many neurons and all but \(\mathcal O \big(\omega_{f,r}^{-1}(\varepsilon)^{-d}\big)\) weights can be fixed. Finally, we have \(\left\lVert R_\varepsilon \right\rVert_\infty\le \left\lVert f \right\rVert_\infty\).
\end{enumerate}
\end{proposition}
\begin{proof}
Building on the previous results we only have to check that there is a triangulation \(\mathcal T\) with fineness at most \(w_{f, r}^{-1}(\varepsilon), k_{\mathcal T} < \infty\)\footnote{One can count the neighboring points and show \(k_{\mathcal T} = (d+1)!\).} such that \([-r, r]^d\) is the union of simplices and 
\[2^d \cdot \left\lceil\frac{\sqrt d\cdot r}{w_{f, r}^{-1}(\varepsilon)}\right\rceil^d\]
vertices in \([-r, r]^d\). We obtain this triangulation by scaling the standard triangulation by
\[r\cdot\left\lceil\frac{\sqrt d\cdot r}{w_{f, r}^{-1}(\varepsilon)}\right\rceil^{-1},\]
for which the properties are easily verified.
\end{proof}

This result can easily be rewritten for Lipschitz continuous functions as the Lipschitz continuity controls the modulus of continuity. We obtain the following approximation result. 

\setcounter{theorem}{3}
\begin{proposition}[Universal approximation under Lipschitz condition]\label{app:uvreplip}
Let \(d, m\in\mathbb N\) and \(r>0\) and let \(f\colon\mathbb R^d\to\mathbb R^m\) be Lipschitz continuous. Then for every \(\varepsilon>0\) there is a ReLU network \(R_\varepsilon\) with parameters \(\theta_\varepsilon\) that satisfies the following:
\begin{enumerate}
\item \emph{Approximation:} It holds that
\(\sup_{x\in[-r, r]^d} \left\lVert f(x) - R_\varepsilon(x) \right\rVert\le\varepsilon.\)
\item \emph{Complexity bounds:} The network has depth \(\big\lceil \log_2((d+1)!)\big\rceil +2\)
, \(\mathcal O\left(r^d\varepsilon^{-d}\right)\) many neurons and all but \(\mathcal O \left(r^d\varepsilon^{-d}\right)\) weights can be fixed. Finally, if \(\left\lVert f \right\rVert\) is bounded by \(c\) so is \(\left\lVert R_\varepsilon \right\rVert\).
\end{enumerate}
\end{proposition}

\setcounter{theorem}{10}

\section{Error estimate for perturbed Euler schemes}

First, we need to introduce the notion of weak solutions of ordinary differential equations.

\begin{definition}[Weak solutions]
Let \(f\colon[0, 1]\times\mathbb R^d\to\mathbb R^d\) be a Carath\'{e}odory function, i.e., measurable in the first and continuous in the second argument and let further \(x_0\in\mathbb R^d\). Then we say \(x\colon[0, 1]\to\mathbb R^d\) is a \emph{weak solution} of the differential equation
\[\partial_tx(t) = f(t, x(t)), \quad x(0) = x_0\]
if it satisfies
\[x(t) = x_0 + \int_0^t f(s, x(s))\mathrm ds \quad \text{for all } t\in [0, 1]. \]
The integral on the right hand side can be interpreted as a componentwise Lebesgue integral where the Carath\'{e}odory condition ensures the measurability. Further, we call \(x\colon[0,1]\times \mathbb R^d\to \mathbb R^d\) the \emph{space-time solution} of the ODE with right hand side \(f\) if it solves
\[\partial_t x(t, y) = f(t, x(t, y)), \quad x(0, y) = y.\]
\end{definition}

The well posedness of ordinary differential equations in the weak sense can be proved just like the well posedness results from the classical theory. In particular, a global solution \(x\colon[0, 1]\to\mathbb R^d\) exists for every initial value \(x_0\in\mathbb R^d\) if \(f(t, \cdot)\) is bounded and Lipschitz continuous for almost all \(t\) with integrable uniform norm and Lipschitz constant \cite[see][]{younes2010shapes}. We denote the space of those functions which are also Bochner-measurable\footnote{See \cite{diestel1977vector}; there such functions are called strongly measurable.} by \(L^1([0, 1]; \mathcal C^{0, 1}_b(\mathbb R^d; \mathbb R^d))\).

\begin{definition}[Euler discretisation]\label{Euflo}
Let \(0 = t_0<\dots<t_n=1\) be a partition of the unit interval and \(x_0\in\mathbb R^d\). Let \(f\colon[0, 1]\times\mathbb R\to\mathbb R\) be an arbitrary Carathéodory function. Then we define the \emph{Euler discretisation or Euler scheme to the right hand side \(f\), initial value \(x_0\) and with respect to the partition \((t_0, t_1, \dots, t_n)\)} via
\[x^n(0)\coloneqq x_0, \quad \text{and } x^n(t_{i+1}) = x^n(t_i) + (t_{i+1} - t_i) f(t_i, x^n(t_i)) \]
and linearly in between.
\end{definition}

It is important to note that the Euler discretisation \(x^n\) satisfies the integral equation
\begin{equation*}
x^n(t) = x_0 + \int\limits_0^t \gamma(s)\mathrm ds \quad \text{for all } t\in [0, 1],
\end{equation*}
where 
\[\gamma(t)\coloneqq \sum_{i=0}^{n-1}\chi_{[t_i, t_{i+1})}f(t_i, x(t_i)) .\] 

\begin{lemma}[Generalised Gr{\"o}nwall inequality]\label{gengron}
Let \(x_0, y_0\in \mathbb R^d\) and let \(\gamma_0, \gamma_1\in L^1([0, 1]; \mathbb R^d)\)\footnote{i.e., their norms are integrable; see \cite{diestel1977vector} for an introduction to vector valued integration.} and let \(x\) and \(y\) satisfy the integral equations
\[x(t) = x_0 + \int\limits_{0}^t \gamma_1(s)\mathrm ds \quad \text{and } y(t) = y_0 + \int\limits_{0}^t \gamma_2(s)\mathrm ds \quad\text{for all } t\in [0, 1]. \]
Assume now that there are non negative functions \(\alpha, \beta\in L^1([0, 1]])\) such that
\[\left\lVert \gamma_1(t) - \gamma_2(t) \right\rVert\le \alpha(t) + \beta(t)\cdot \left\lVert x(t) - y(t) \right\rVert \quad \text{for all } t\in [0, 1].\]
Then we have
\[\left\lVert x(t) - y(t) \right\rVert\le c\cdot\left( \left\lVert x_0 - y_0 \right\rVert + \left\lVert \alpha \right\rVert_{L^1(I)}\right)\quad \text{for all } t\in [0, 1] ,\]
where we can choose
\[c = 1 + \left\lVert \beta \right\rVert_{L^1([0, 1])} \cdot \exp(\left\lVert \beta \right\rVert_{L^1([0, 1])}).\]
\end{lemma}
\begin{proof}
For \(t\ge t_0\) we compute
\begin{equation*}
\begin{split}
 \left\lVert x(t) - y(t) \right\rVert \le & \; \left\lVert x_0 - y_0 \right\rVert + \int\limits_{t_0}^t \left\lVert \gamma_1(s) - \gamma_2(s) \right\rVert\mathrm ds \\
 \le & \; \left\lVert x_0 - y_0 \right\rVert + \int\limits_{t_0}^t \alpha(s)\mathrm ds + \int\limits_{t_0}^t \beta(s) \cdot \left\lVert x(s) - y(s) \right\rVert\mathrm ds. 
\end{split}
\end{equation*}
An application of Gr{\"o}nwall's inequality yields the assertion.\footnote{For a general version of Gr{\"o}nwall's inequality we refere to Theorem 1.2.8 in \cite{qin2017analytic}.} For \(t\le t_0\) the computation follows in analogue way or by reflection.
\end{proof}

\begin{remark}\label{remarkthatsolvesitall}
If \(\left\lVert f \right\rVert\) is bounded by \(c\) we obtain the growth estimate
\[\left\lVert x(t) \right\rVert\le \left\lVert x_0 \right\rVert + c.\]
Further, this estimate holds also for all Euler discretisations of \(f\).
\end{remark}

\begin{proposition}[Continuity of solution map]\label{contsolmap}
Let \(x_0, y_0\in \mathbb R^d\) and let \(f, g\colon [0, 1]\times \mathbb R^d\to \mathbb R^d\) be Carathéodory functions such that \(f(t, \cdot)\) is Lipschitz continuous with constant \(h(t)\) for \(t\in[0, 1]\) where \(h\in L^1([0, 1])\). Further, let \(f-g\in L^1([0, 1]; L^\infty(\mathbb R^d; \mathbb R^d))\)\footnote{i.e., the uniform distance \(\left\lVert f(t, \cdot) - g(t, \cdot) \right\rVert_\infty\) is integrable over \([0, 1]\).} and let \(x, y\colon [a, b]\to \mathbb R^d\) be weak solutions to the differential equations
\[\partial_t x(t) = f(t, x(t)), \quad x(t_0) = x_0 \quad \text{and } \partial_t y(t) = g(t, y(t)), \quad y(t_0) = y_0. \]
Then we have
\begin{equation}\label{contest}
\sup_{t\in [0, 1]}\left\lVert x(t) - y(t) \right\rVert \le c\cdot \left( \left\lVert x_0 - y_0 \right\rVert + \left\lVert f - g \right\rVert_{L^1([0, 1]; L^\infty (\mathbb R^d; \mathbb R^d))}\right),
\end{equation}
where the constant \(c\) only depends on \(\left\lVert h \right\rVert_{L^1([0, 1])}\).
\end{proposition}
\begin{proof}
We only need to check the requirements of the previous result. We recall that \(x\) and \(y\) solve the integral equations associated to the ODEs and hence obtain for \(t\in I\) 
\begin{equation*}
\begin{split}
\left\lVert \gamma_1(t) - \gamma_2(t) \right\rVert = & \; \left\lVert f(t, x(t)) - g(t, y(t)) \right\rVert \\
\le & \; \left\lVert f(t, x(t)) - f(t, y(t)) \right\rVert + \left\lVert f(t, y(t)) - g(t, y(t)) \right\rVert \\
\le & \; h(t)\cdot \left\lVert x(t) - y(t) \right\rVert + \left\lVert f(t, \cdot) - g(t, \cdot) \right\rVert_{\infty}.
\end{split}
\end{equation*}
\end{proof}

Later we will perceive residual networks as an perturbed Euler approximation of an ordinary differential equation. To show convergence of those we provide an error estimate for such perturbations, namely we replace the direction \(f(t_i, x^n(t_i))\) of the Euler approximation \(x^n\) on \([t_i, t_{i+1})\) by \(z_i \approx f(t_i, x^n(t_i))\). 

\begin{proposition}[Error estimate for perturbed Euler schemes]\label{eulest}
Let \(f\colon[0, 1]\times \mathbb R^d\to \mathbb R^d\) be a Carathéodory function such that \(f(t, \cdot)\) is Lipschitz with constant \(h(t)\) where \(h\in L^1([0, 1])\). Let now \(x_0\in \mathbb R^d\) and \(x\colon[0, 1]\to \mathbb R^d\) be the weak solution to
\[\partial_t x(t) = f(t, x(t)) \quad \text{and } x(0) = x_0. \]
Fix \(z_0, \dots, z_{n-1}\in \mathbb R^d\) and set \(t_i\coloneqq i/n\) as well as
\(\gamma\coloneqq \sum_{i=0}^{n-1}\chi_{[t_i, t_{i+1})}z_i\).
Let \(x^n\colon [0, 1]\to \mathbb R^d\) satisfy the integral equation
\[x^n(t) = x_0 + \int\limits_0^t \gamma(s)\mathrm ds \quad \text{for all } t\in [0, 1]. \]
Assume that we have
\(\left\lVert z_i - f(t, x^n(t_i)) \right\rVert\le \varepsilon \quad \text{for all } t\in [t_i, t_{i+1}), i = 0, \dots, n-1\)
as well as
\(\left\lVert z_i \right\rVert\le c\quad \text{for all } i = 0, \dots, n-1\).
Then we have
\[\left\lVert x^n(t) - x(t) \right\rVert\le \tilde c\cdot\left( \varepsilon + \frac{c}{n} \cdot \left\lVert h \right\rVert_{L^1([0, 1])}\right),\]
where \(\tilde c\) only depends on \(\left\lVert h \right\rVert_{L^1([0, 1])}\).
\end{proposition}
\begin{proof}
Once more we will use Lemma \ref{gengron} with obvious choices of \(\gamma_1\) and \(\gamma_2\). For \(t\in [t_i, t_{i+1})\) we estimate
\begin{equation*}
\begin{split}
\left\lVert \gamma_1(t) - \gamma_2(t) \right\rVert = & \; \left\lVert z_i - f(t, x(t)) \right\rVert \\
\le & \; \left\lVert z_i - f(t, x^n(t_i)) \right\rVert + \left\lVert f(t, x^n(t_i)) - f(t, x(t)) \right\rVert \\ 
\le & \; \varepsilon + \left\lVert f(t, x^n(t_i)) - f(t, x^n(t)) \right\rVert + \left\lVert f(t, x^n(t))- f(t, x(t)) \right\rVert \\ 
\le & \; \varepsilon + h(t) \cdot\left\lVert x^n(t_i) - x^n(t) \right\rVert + h(t) \cdot\left\lVert x^n(t) - x(t) \right\rVert \\
\le & \; \varepsilon + \frac{c}n \cdot h(t) + h(t) \cdot\left\lVert x^n(t) - x(t) \right\rVert.
\end{split}
\end{equation*}
\end{proof}

\section{Proofs of the main results}

Let us quickly recall our definition of residual networks. Let \(R_1, \dots, R_{n}\colon\mathbb R^d\to\mathbb R^d\) be neural networks. The resulting ResNet \(x^n\colon[0, 1]\times\mathbb R^d\to\mathbb R^d\) is defined via
\[x^n(0, y) \coloneqq y, \quad x^n(t_{k+1}, y)\coloneqq x^n(t_k, y) + n^{-1}\cdot R_{k+1}(x^n(t_k, y))\]
for \(k = 0, \dots, n-1\) and linearly in between.

It is clear from the definition that ResNets are in fact Euler approximations to the piecewise constant right hand side \(f\colon[0, 1]\times\mathbb R^d\to\mathbb R^d\) which is defined via
\[f(t, x) \coloneqq \sum_{i=0}^{n-1}\chi_{[t_i, t_{i+1})}(t) R_{i+1}(x) \quad \text{for all } t\in [0, 1], x\in\mathbb R^d ,\]
where \(t_i\coloneqq i/n\).
We will use the error estimate on perturbed Euler schemes established in the previous chapter to show that by letting \(n\to\infty\) and increasing the expressivity of the networks \(R_i\) ReLU ResNets are able to approximate space-time solutions of arbitrary right hand sides.

\begin{lemma}
Let \(f\in L^1([0, 1]; \mathcal C^{0, 1}_b(\mathbb R^d; \mathbb R^d))\)
and let \(x\) be the space-time solution of the ODE with right hand side \(f\). Then for every \(\varepsilon>0\) there is \(n\in\mathbb N\) and \(g\in L^1([0, 1]; C^{0, 1}_b(\mathbb R^d; \mathbb R^d))\) that is constant on all intervals of the form \([i/n, (i+1)/n)\) such that the space-time solution \(\tilde x\) to \(g\) satisfies
\[\left\lVert x(t, y) - \tilde x(t, y) \right\rVert\le\varepsilon\quad \text{for all } t\in[0, 1], y\in \mathbb R^d.\]
\end{lemma}
\begin{proof}
By standard Bochner theory \cite[see][]{arendt2011vector} the continuous functions are dense in 
\[L^1\left([0, 1]; \mathcal C^{0, 1}_b\left(\mathbb R^d; \mathbb R^d\right)\right).\]
However, continuous functions can approximated arbitrarily well by functions that are constant on intervals of equal length. Now the continuity estimate \eqref{contest} yields the assertion.
\end{proof}

\setcounter{theorem}{1}
\begin{theorem}[Space-time approximation with ResNets]\label{app:uniapproxflow}
Let \(d\in\mathbb N\) and 
\[f\in L^1\left([0, 1]; \mathcal C^{0, 1}_b\left(\mathbb R^d; \mathbb R^d\right)\right)\]
and let \(x\) be the space-time solution to \(f\). Then for every compact set \(K\subseteq\mathbb R^d\) and \(\varepsilon>0\) there is a ReLU ResNet \(\tilde x\) such that
\[ \left\lVert \tilde x(t, y) - x(t, y) \right\rVert \le \varepsilon \quad \text{for all } t\in [0, 1], y\in K.\]
\end{theorem}
\begin{proof}
By the previous lemma we can without loss of generality assume that \(f\) is constant on the intervals \([i/n, (i+1)/n)\) for some \(n\in\mathbb N\). We note that since \(f\) is piecewise constant with values in \(\mathcal C^{0, 1}_b(\mathbb R^d; \mathbb R^d)\) there is \(c>0\) such that
\begin{align*}\label{lingrow4}
\left\lVert f(t, x) \right\rVert \le c \quad \text{for all } t\in [0, 1], x\in\mathbb R^d.
\end{align*}
It suffices to show the statement for the compact set \(K=\overline{B_N}\) where \(B_N\) denotes the ball of radius \(N\) around the origin. By Remark \ref{remarkthatsolvesitall} we have \(x(t, y) \in \overline{B_M}\) for every \(t\in [0, 1], y\in\overline{B_N}\) where \(M=N+c\). 

Let now \(\varepsilon >0\), then the universal approximation result \ref{app:uniapprox} for ReLU networks yields the existence of ReLU networks \(R_0, \dots, R_{n-1}\colon\mathbb R^d\to\mathbb R^d\) with parameters \(\theta_0, \dots, \theta_{n-1}\) such that
\begin{equation}\label{spatialapprox}
\left\lVert f(t, y) - R_i(y) \right\rVert \le\varepsilon \quad \text{for all } y\in \overline{B_M}, t\in [i/n, (i+1)/n).
\end{equation}
as well as \(\left\lVert R_i \right\rVert \le c\). 
Further, we choose \(k\in\mathbb N\) such that
\begin{equation}\label{kest}
\frac{c}{ kn} \cdot \left\lVert f \right\rVert_{L^1([0, 1]; \mathcal C^{0, 1}_b(\mathbb R^d; \mathbb R^d))} \le \varepsilon.
\end{equation}
Let now \(\tilde x\) be the ReLU ResNet with parameters
\begin{equation}\label{app:resnet}
\left(\theta_0,\dots, \theta_0, \theta_1, \dots, \theta_1, \dots, \theta_{n-1}, \dots, \theta_{n-1}\right),
\end{equation}
where each network \(\theta_i\) is included \(k\) times. Now we aim to apply Proposition \ref{eulest} and hence check its requirements and fix \(y\in\overline{B_N}\) and denote \(x(t, y), \tilde x(t, y)\) with \(x(t)\) and \(\tilde x(t)\) respectively and again Remark \ref{remarkthatsolvesitall} yields \(\tilde x(t)\in \overline{B_M}\) for all \(t\in[0, 1]\). In order to use the notation from the proposition we set \(t_i\coloneqq i/(kn)\) and \(z_i\coloneqq R_j(\tilde x(t_i))\) for \(i = kj, \dots, k(j+1)-1\) and obtain 
\[\tilde x(t) = y + \int\limits_0^t \gamma(s)\mathrm ds \quad \text{for } \gamma = \sum\limits_{i=0}^{kn-1} \chi_{[t_i, t_{i+1})} z_i.\]
Further, it holds that
\[\left\lVert z_i - f(t, \tilde x(t_i)) \right\rVert\le \varepsilon \quad \text{for all } t\in [t_i, t_{i+1}), i = 0, \dots, kn-1\]
as well as \(\left\lVert z_i \right\rVert\le c\). Now Proposition \ref{eulest} yields 
\begin{align*}
\left\lVert \tilde x(t) - x(t) \right\rVert \le & \; \tilde c \cdot\left( \varepsilon + \frac c{kn} \cdot \left\lVert f \right\rVert_{L^1([0, 1]; \mathcal C^{0, 1}_b(\mathbb R^d; \mathbb R^d))}\right) \\
\le & \;2\tilde c \cdot \varepsilon \quad \text{for all } t\in [0, 1], 
\end{align*}
where \(\tilde c\) only depends on \(\left\lVert f \right\rVert_{L^1([0, 1]; \mathcal C^{0, 1}_b(\mathbb R^d; \mathbb R^d))}\) and not on \(y\in \overline{B_N}\).
\end{proof}

The universal approximation theorem presented above is of qualitative nature since it does not give any estimates on the complexity of the residual network needed to approximate a flow up to a certain precision. This is due to the fact that we work with density results for continuous functions in the Bochner space \(L^1([0, 1]; \mathcal C^{0, 1}_b\left(\mathbb R^d; \mathbb R^d\right))\).
In the proof above one could also assume that \eqref{kest} holds for \(k=1\) since \(f\) is also piecewise constant on the intervals \([i/(kn), (i+1)/(kn))\). However, we wanted to separate the approximation procedures in space and in time. More precisely, if \(f\) is (almost) constant in time, \eqref{spatialapprox} can be achieved with little \(n\) and hence the constructed ResNet \eqref{resnet} shares a lot of weights. This observation could be used to explore approximation capabilities of ResNets with shared weights under different spatial and temporal regularity of the right hand side \(f\). 

We use analogue arguments to establish estimates on the number and complexity of residual blocks required to approximate space-time solutions of ODEs with Lipschitz continuous right hand side \(f\).

\begin{theorem}[Space-time approximation with complexity bounds]\label{app:uniapproxcomplex}
Let \(d\in\mathbb N\), \((r_n)_{n\in\mathbb N}\subseteq(0, \infty)\) be a sequence convergent to \(\infty\) and let \(f\colon[0, 1]\times\mathbb R^d\to\mathbb R^d\) be a bounded and 
Lipschitz continuous function. Let \(x\colon[0, 1]\times\mathbb R^d\to\mathbb R^d\) be the space-time solution of the ODE with right hand side \(f\). 
Then for every \(n\in\mathbb N\) there is a ReLU ResNet \(x^n\) with parameters \(\theta^n = (\theta^n_1, \dots, \theta^n_n)\) such that the following are satisfied:
\begin{enumerate}
\item \emph{Approximation:} For every compact set \(K\subseteq\mathbb R\) it holds
\[\sup_{t\in[0, 1], y\in K}\left\lVert x^n(t, y) - x(t, y) \right\rVert \in \mathcal O(n^{-1}).\]
\item \emph{Complexity bounds:} Every residual block \(\theta^n_k\) has depth \(\big\lceil \log_2((d+1)!)\big\rceil +2\) and satisfies
\[N(\theta^n_k) \in \mathcal O\left(r_n^d n^{d}\right).\]
Finally, all but \(\mathcal O\left(r_n^d n^{d}\right)\) weights can be fixed. 
\end{enumerate}
\end{theorem}
\begin{proof}
We fix \(n\in\mathbb N\) and set \(t_i\coloneqq i/n\). Let \(R^n_i\) be ReLU networks of asserted complexity that approximate \(f(t_i, \cdot)\) on \([-r_n, r_n]^d\) up to \(n^{-1}\) which exist by Proposition \ref{app:uvreplip}. Let \(x^n\) be the ReLU ResNet with residual blocks \(R^n_1, \dots, R^n_n\). We fix \(N>0\) and will show 
\[\sup_{t\in [0, 1], y\in \overline{B_N}} \left\lVert x^n(t, y) - x(t, y) \right\rVert\in \mathcal O(n^{-1}) \quad \text{for } n\to\infty\]
through an application of Proposition \ref{eulest}. Since \(f\) is bounded, there is \(c>0\) such that 
\[\left\lVert f(t, y) \right\rVert\le c \quad \text{for all } t\in [0, 1], y\in \mathbb R^d\]
and hence the functions \(R^n_i\) satisfy this as well. Setting \(M\coloneqq N+c\), Remark \ref{remarkthatsolvesitall} yields 
\[x(t, y), x^n(t, y)\in\overline{B_M}\quad \text{for all } t\in[0, 1], y\in\overline{B_N}. \]
Now we fix \(y\in\overline{B_N}\) and write \(x(t), x^n(t)\) for \(x(t, y)\) and \(x^n(t, y)\) respectively; to keep to the notation of the error estimate for Euler schemes, we set \(z_i\coloneqq R^n_i( x^n(t_i))\). For \(n\ge M\) we obtain
\begin{equation*}
\begin{split}
\left\lVert z_i - f(t, x^n(t_i)) \right\rVert = & \; \left\lVert R^n_i(x^n(t_i)) - f(t, x^n(t_i)) \right\rVert \\
 \le & \; \left\lVert R^n_i(x^n(t_i)) - f(t_i, x^n(t_i)) \right\rVert + \left\lVert f(t_i, x^n(t_i)) - f(t, x^n(t_i)) \right\rVert \\
 \le & \; n^{-1} \cdot\left( 1 + L\right)
\end{split}
\end{equation*}
for all \(t\in [t_i, t_{i+1})\) and \(i = 0, \dots, n-1\) where \(L\) denotes the Lipschitz constant of \(f\). Furthermore, we have \(\left\lVert z_i \right\rVert \le c\) for all \(i=0, \dots, n-1\) and hence Proposition \ref{eulest} completes the proof. 
\end{proof}

\end{document}